\documentclass{article}

 \usepackage[main, final]{neurips_2026}

\usepackage[utf8]{inputenc}
\usepackage[T1]{fontenc}
\usepackage{hyperref}
\usepackage{url}
\usepackage{booktabs}
\usepackage{longtable}
\usepackage{amsfonts}
\usepackage{nicefrac}
\usepackage{microtype}
\usepackage{xcolor}

\usepackage{amsmath}
\usepackage{amssymb}
\usepackage{mathtools}
\usepackage{amsthm}
\usepackage[capitalize,noabbrev]{cleveref}
\usepackage[textsize=tiny]{todonotes}
\usepackage{enumitem}
\usepackage{subcaption}
\usepackage{wrapfig}
\usepackage[english]{babel}
\usepackage[autostyle, english = american]{csquotes}

\usepackage{thmtools}
\usepackage{thm-restate}
\usepackage{algorithm}
\usepackage{algpseudocode}

\theoremstyle{plain}
\newtheorem{theorem}{Theorem}[section]

\theoremstyle{definition}
\newtheorem{definition}[theorem]{Definition}

\theoremstyle{remark}

\newcommand{\E}{\mathbb{E}}

\newcommand{\aprime}{{a^\prime}}

\DeclareMathOperator{\nie}{NIE}

\DeclareMathOperator{\pie}{PIE}
\DeclareMathOperator{\te}{TE}

\newcommand{\interaction}{\mathrm{INT}}

\title{The Curse of Multiple Mediators: \\Hidden Interaction Effects in Activation Patching}

\author{%
  Sankaran Vaidyanathan \\
  University of Massachusetts Amherst\\
  \texttt{sankaranv@cs.umass.edu} \\
  \And
  David Arbour \\
  Adobe Research\\
  \texttt{arbour@adobe.com} \\
  \And
  Aaron Mueller \\
  Boston University\\
  \texttt{amueller@bu.edu} \\
  \And
  Scott Niekum \\
  University of Massachusetts Amherst\\
  \texttt{sniekum@cs.umass.edu} \\
  \And
  David Jensen \\
  University of Massachusetts Amherst\\
  \texttt{jensen@cs.umass.edu} \\
}

\begin{document}
\maketitle

\begin{abstract}

Activation patching is a primary tool in mechanistic interpretability. It attributes causal responsibility for a model's behavior to each of its individual components by estimating its \textit{natural indirect effect (NIE)}. Re-deriving the activation patching estimand from causal mediation analysis, we find that the NIE does not solely capture the causal effect through the specific component. It also contains \textit{interaction effects (INT)} that measure how much the component's causal effect itself depends on the states of other components in the model. A natural response may be to try to eliminate INT by adjusting the estimator or unit of analysis, but each of these potential remedies has predictable failure modes. We demonstrate these failure modes in the GPT-2 IOI circuit; components whose causal importance is conditional on the state of other components are either invisible or artificially inflated. We prove that INT scales with the distance between clean and patched component activations, is negligible when the model is nearly affine between them, and decomposes combinatorially into pairwise and higher-order group interactions. Despite its inevitability, INT is not a nuisance to be eliminated, but rather a diagnostic for interpretability studies. Its individual and group-level magnitude and sign signal when causal conclusions are prompt-dependent, and when greedy NIE-based component ranking will miss mechanisms only discoverable through combinatorial search.
\end{abstract}

\section{Introduction}

Mechanistic interpretability aims to identify which internal components in a neural network implement a given behavior or capability, and to characterize the causal role of each identified component \citep{olah2020zoom, elhage2021mathematical}. 
Component localization is grounded in \textit{causal mediation analysis}, particularly in isolating how much of the total causal effect flows through each component by computing its \textit{natural indirect effect (NIE)} \citep{pearl2001direct}. In neural networks, \textit{activation patching} \citep{vig2020genderbias, geiger2021causal} operationalizes the NIE by substituting a component's activations with counterfactual values, interpreting the resulting change in output as that component's causal contribution. Activation patching and its derivatives \citep{goldowskydill2023localizingmodelbehaviorpath, syed2024attribution, hanna2024have} are the field's de facto standard instruments for causal attribution, ranking components, and evaluating circuits.

However, activation patching has known failure modes that appear to contradict its claim as a measure of causal importance, most notably missing components that are functionally important for a given task \citep{mueller2024missed}. The circuit for indirect object identification (IOI) \citep{wang2023interpretability} is the canonical example of such false negatives; it contains backup name mover heads that only show significant causal effects when ablating all other components that performed the same function. These were only found through combinatorial search guided by prior hypotheses about the components' function---an approach that does not scale and is no longer standard practice. Circuit evaluations have also encountered other counterintuitive issues, including high variability across ablation methods, prompt distributions, and data subsamples \citep{miller2024transformer, meloux2026variance}, non-monotonicity \citep{shi2024hypothesis}, and low-scoring circuits that overlap highly with ground truth \citep{hanna2024have}. While some of these issues have been explained by appealing to the complexity of the underlying model, or by how interventions are constructed \citep{makelov2023subspacelookingforinterpretability, grant2026addressing}, the assumptions of the NIE itself have rarely been questioned.

Since \citet{pearl2001direct} proposed the NIE, the causal inference literature has identified failure cases where the NIE does not correctly isolate the effect through a specific component \citep{andrews2021crossworldindependence, vanderweele2014unification}. The structure underlying such failures is the recanting witness \citep{avin2005identifiability}: a variable that lies on both the target path and a bypass route around the mediator, in whose presence the effect along the target path cannot be recovered from standard interventional experiments. The transformer residual stream guarantees this condition for every mediator by construction: the skip connection at each layer routes information directly from earlier components to later ones, around whatever activation is being patched \citep{elhage2021mathematical}. This gives us principled reason to suspect that the NIE may not isolate component-specific contributions in transformers. 

We prove that the NIE in a transformer contains \textit{interaction effects (INT)} \citep{vanderweele2013ThreewayDecomposition}, which measure how much a component's causal effect itself depends on the state of the rest of the model. INT arises in transformers because activation patching holds one component fixed at its counterfactual value while the rest of the model runs on the clean input; the output is then a nonlinear function of both, and the portion of the NIE that cannot be attributed to either alone is INT. Activation patching computes $\nie = \pie + \interaction$, where PIE is the pure indirect effect: the causal effect measured solely through paths that contain the component, which the NIE was assumed to be recovering. INT has been present in every activation patching study. 

In this paper, we study the theoretical properties and empirical consequences of interaction effects in transformer activation patching. Our key findings include the following:

\begin{enumerate}[label=(\alph*), leftmargin=0.65cm]

\item NIE and PIE induce substantially different rankings over component importance across models and datasets (\cref{fig:int-measurement}), with rank correlation as low as $\rho = 0.15$ on IOI in GPT-2 medium.

\item INT scales with the distance between clean and patched component activations, and is zero when the downstream layers are affine between them (\cref{thm:no-int-single}). Empirically, INT is very small in tasks where clean and counterfactual prompts are highly similar, and in most compiled Tracr \citep{lindner2023tracr} models used in interpretability benchmarks. 

\item In multi-component patching estimands such as faithfulness scores, INT decomposes combinatorially into individual, pairwise and group-level interactions (\cref{thm:multi-noising-decomp}). 

\item Interaction effects explain known failure cases of ranking by NIE scores in the GPT-2 IOI circuit (\cref{sec:ioi-roles}). Backup Name Mover heads were found to have low NIE but take over the role of the Name Mover heads when those are knocked out. Their NIE is suppressed by negative INT.

\item Removing INT and ranking by PIE introduces ranking failures of its own (\cref{sec:ioi-failures}). On the IOI task in GPT-2, PIE is near-zero for context-specific components such as Duplicate Token and Previous Token heads; their importance depends on other heads being simultaneously active.
 
\end{enumerate}

INT is a fundamental feature of causal mediation, growing in complexity with the size and nonlinearity of modern transformers. Its presence in the NIE propagates downstream to ranking and evaluation procedures, and it provides a unifying theory for several known pathologies in the field. INT is measurable at many levels of granularity and enables practitioners to identify the special cases in which their conclusions can be trusted without the need for combinatorial search. That is not a limitation of activation patching; it is a verifiable statement of its assumptions, and a new lens on model mechanisms that single-component effects cannot see.

\begingroup
\renewcommand\thefootnote{}
\let\theHfootnote\relax
\footnotetext{Code available at: \url{https://github.com/sankaranv/mech-interp-mediation-analysis}}
\endgroup

\section{Background}
\label{sec:background}

Causal mediation analysis (CMA) decomposes the total causal effect of a treatment $A$ on an outcome $Y$ into contributions from specific causal pathways. For treatment $a$ and baseline $\aprime$, the total causal effect is $\te(a, \aprime) = \E[Y(a)] - \E[Y(\aprime)]$. When a \textit{mediator} $M$ lies on some but not all paths from $A$ to $Y$, the natural indirect effect (NIE) measures how much of the total effect passes through $M$.

\begin{definition}[Natural Indirect Effect]
\label{def:nie}
For treatment $a$ and baseline $\aprime$, the natural indirect effect \citep{pearl2001direct} through mediator $M$ is
$$\nie(a, \aprime) = \E[Y(a, M(a))] - \E[Y(a, M(\aprime))]$$
where $Y(a, M(\aprime))$ is the outcome under treatment $a$ with $M$ fixed at its baseline value $M(\aprime)$.
\end{definition}

The NIE is intended to isolate the effect through $M$ and block all paths that bypass it. \citet{avin2005identifiability} formalizes this target as the \textit{path-specific effect (PSE)}: the causal effect propagated along a designated subgraph of the model.

In transformers, mediators are typically activations of model components such as attention heads, MLPs, or sparse autoencoder latents \citep{mueller2026quest}. Because the residual stream introduces a skip connection $x_{\ell-1} \to x_\ell$ at each layer, every mediator has a bypass path, which for an attention head also includes the other heads in its layer. The computation of NIE in transformers goes by various names in the literature such as activation patching \citep{heimersheim2024patching}, causal tracing \citep{meng2022locating}, and interchange intervention \citep{geiger2021causal}; we use \textit{patching} throughout this paper. 

The treatment $a$ and baseline $\aprime$ are conventionally called the \textit{clean} and \textit{corrupted} inputs. This distinction gives rise to two patching estimators.

\begin{definition}[Patching Estimators]
\label{def:patching}
Let $M$ be a set of nodes with treatment (clean) activation $M(a)$ and baseline (corrupted) activation $M(\aprime)$. Activation patching yields the following estimators:
\begin{enumerate}[label=\normalfont(\alph*),nosep]
  \item \emph{Denoising}: run the model on input $\aprime$, then replace $M(\aprime)$ with $M(a)$.
  \item \emph{Noising}: run the model on input $a$, then replace $M(a)$ with $M(\aprime)$.
\end{enumerate}
\end{definition}

The two operations differ in which input provides the baseline: denoising restores components to their clean values within a forward pass under the corrupted input; noising corrupts the components within a forward pass under the clean input. Noising is currently the default approach for measuring component importance in transformers \citep{conmy2023towards, hanna2024have, marks2025sparse}.

\section{Interaction Effects in Transformers}
\label{sec:interaction-single}

Activation patching measures each component's causal contribution by computing its natural indirect effect, but the NIE is not a function of that component's activation alone. Every component operates alongside a bypass: the residual skip connection and same-layer components route information around it to downstream layers by construction~\citep{elhage2021mathematical}. While patching the component's activation changes what it contributes to the residual stream, it leaves the bypass on its original input, and downstream nonlinearities couple both contributions in a way that is difficult to attribute to either alone \citep{avin2005identifiability, andrews2021crossworldindependence}. Proposition~\ref{prop:single-mediator-decomp} identifies how the bypass factors into the causal effect that the NIE was assumed to isolate, and what it actually measures.

\begin{restatable}[Single-Mediator Interaction]{proposition}{singleMediatorInt}
\label{prop:single-mediator-decomp}
Consider a single component $M$ at layer $\ell$, and let $B_M$ be the set of edges that bypass $M$ in the computation graph.

\begin{enumerate}[label=\normalfont(\alph*),nosep]
    \item Denoising computes $\pie = Y(a', M(a)) - Y(a')$, the pure indirect effect, which is the path-specific effect through the subgraph in which the bypass edges $B_M$ are blocked.
    \item Noising computes
          $\nie = \pie + \interaction$, where
          \begin{equation}
            \label{eq:int}
            \interaction = Y(a) - Y(a', M(a)) - Y(a, M(a')) + Y(a')
          \end{equation}
          is the \textbf{mediated interaction} between $M$ and its bypass $B_M$.
    \item The total effect decomposes as $\mathrm{TE} = \pie + \mathrm{PDE} + \interaction$, where $\mathrm{PDE} = Y(a, M(a')) - Y(a')$ is the pure direct effect, i.e., the path-specific effect through the subgraph in which the edge from $M$ to the residual stream is blocked.
\end{enumerate}
\end{restatable}

Proposition~\ref{prop:single-mediator-decomp} establishes that noising and denoising are not symmetrical operations, and actually estimate distinct causal quantities, NIE and PIE\footnote{\citet{pearl2001direct} defines the NIE with the treatment held at $\aprime$;
\citet{vanderweele2013ThreewayDecomposition}, following \citet{robins1992identifiability}, calls that quantity the pure indirect effect (PIE), and defines the NIE with treatment held at $a$. We follow this convention, since noising estimators are more commonly used in practice.}. The PIE is referred to as ``pure'' because it corresponds to the surgical removal of bypass edges; it is the causal effect that exclusively flows through the given component. The NIE additionally contains an \textit{interaction} term \citep{vanderweele2013ThreewayDecomposition} that captures how much the component's contribution to the outcome depends on what the bypass is simultaneously carrying. For instance, a positive interaction indicates that the rest of the model amplifies the effect of the component. Ignoring INT does not eliminate it. Activation patching absorbs it entirely into the estimate it computes, thereby inflating or deflating it by the full magnitude of INT.

Interaction is a recognized quantity in causal mediation analysis~\citep{vanderweele2015explanation}, and expected to arise whenever treatment and mediator share downstream causal pathways. For example, in a study of the interaction between age and alcohol consumption in their association with \textit{H.~pylori} infection, alcohol consumption was found to be causative for younger individuals but preventative for older adults~\citep{paunio1994alcoholhelicobacter, vanderweele2019interaction}: the direction of the causal effect depended entirely on the state of a second variable. The causal inference literature documents a rich taxonomy of such interaction patterns~\citep{vanderweele2019interaction}, of which sign reversal is only one case. \citet{heimersheim2024patching} identified one such pattern in transformers, attributing the noising-denoising gap to AND/OR circuit topology. Interaction is the underlying phenomenon; in \cref{sec:ioi-case-study} we show that the sign and magnitude of INT indicate a wider range of functional roles than AND/OR circuits alone.

\paragraph{INT and Ranking Divergence.}
The practical consequence of $\interaction \neq 0$ is that noising and denoising rank components differently. \cref{fig:int-measurement} shows the Spearman rank correlation $\rho(\overline{\nie}, \overline{\pie})$ across all attention heads for five tasks. GPT-2 small \citep{radford2019language} was selected since it was used in the original papers for every task, and we additionally include GPT-2 medium and XL and models from the Pythia \citep{biderman2023pythia}, Qwen2.5 \citep{qwen2025qwen25technicalreport}, and Llama3.1 \citep{grattafiori2024llama3herdmodels} families. We use exact activation patching rather than fast approximations such as attribution patching or integrated gradients, since our goal is to study the fundamental properties of NIE and PIE directly.

Rank agreement between NIE and PIE tracks how semantically similar the clean and counterfactual prompts are. At the high end, $\rho = 0.996$ for SVA \citep{marks2025sparse} and $0.919$ for IOI \citep{wang2023interpretability} with symmetric corruption where the name tokens are swapped. These are datasets where the clean and counterfactual prompts are extremely close to each other; for example, an SVA counterfactual changes a single word between singular and plural. Gender Bias \citep{vig2020genderbias} lies in between, at $\rho = 0.603$. On the other hand, $\rho = 0.509$ for IOI with pABC corruption and $0.513$ for Greater Than \citep{hanna2023how}, which use counterfactuals that are more substantially different from the clean prompt. In larger models, $\rho$ on these two tasks is even lower, down to $0.15$ for IOI in GPT-2 medium and $0.10$ for Greater-Than in GPT-2 XL. 

INT as defined in Equation~\ref{eq:int} is a discrete quantity requiring no additional assumptions about the downstream function. Its size is nonetheless controlled by three factors: how far the change in input moves the component, how far it moves its bypass, and how curved the downstream map is between the four forward passes. When clean and counterfactual prompts are close, the first two factors are small and INT is suppressed; when the downstream map is affine, INT
vanishes at any distance.

\begin{figure}[t]
  \centering
  \includegraphics[width=\linewidth]{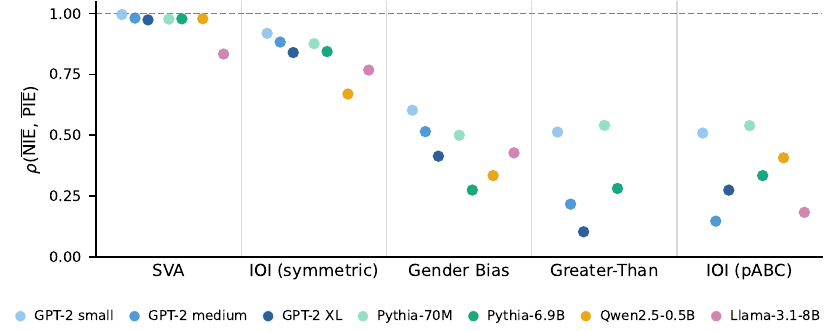}
  \caption{%
    Spearman rank correlation between mean NIE and mean PIE per head across all heads. The dashed line indicates perfect agreement where INT leaves the ranking unchanged. Greater-Than is excluded for Qwen2.5 and Llama3.1; they do not produce the two-digit year ending as a single token.
  }
  \label{fig:int-measurement}
\end{figure}

\begin{restatable}[INT as Curvature]{proposition}{singleMediatorNoInt}
\label{thm:no-int-single}
Let $\delta_M = M(a) - M(a')$ and $\delta_B = B_M(a) - B_M(a')$ be the changes in the outputs of $M$ and its bypass respectively, induced by changing the input from $a$ to $\aprime$. Let $Y = \Psi(z)$ be the output metric as a function of the residual stream $z$ that $M$ writes to, and let $P$ be the parallelogram whose corners are the values of that stream in the four forward passes of \cref{prop:single-mediator-decomp}. When $\Psi$ is twice continuously differentiable,
$$\interaction = \delta_M^\top \bar{H}\, \delta_B
\qquad\text{and}\qquad
|\interaction| \le \|\delta_M\|\,\|\delta_B\|\,\max_{z \in P}\bigl\|\nabla^2\Psi(z)\bigr\|$$
where $\bar{H}$ is the Hessian of $\Psi$ averaged over $P$ and $\|\cdot\|$ is the spectral norm. The interaction is zero, whether or not $\Psi$ is differentiable, under the following sufficient conditions:

\begin{enumerate}[label=\normalfont(\roman*),nosep]
    \item $\delta_M = 0$ or $\delta_B = 0$, i.e., the mediator or the bypass is unaffected by the change in input.
    \item $\Psi$ is additively separable on $P$, i.e., the effects of moving the mediator and moving the bypass are additive: $\Psi(z' + s\,\delta_M + t\,\delta_B) = f_M(s) + f_B(t)$ for all $s, t \in [0, 1]$, where $z'$ is the value of the residual stream under $a'$. This includes when $\Psi$ is affine on $P$.
\end{enumerate}
\end{restatable}

\citet{vig2020genderbias} empirically found that $\nie \approx \pie$ across attention heads on the Winobias dataset, and made a no-interaction assumption in their framework. \cref{thm:no-int-single} explains why: their corruption swaps a single pronoun (\emph{she} to \emph{he}), producing semantically close clean-corrupt pairs, so $\|\delta_M\|$ and $\|\delta_B\|$ are small, and so is the upper bound on $|\interaction|$. Figure~\ref{fig:int-distance} shows how INT grows with perturbation distance for six heads with different functional roles in the IOI task. As the mediator activation is interpolated toward its counterfactual value, INT grows approximately linearly with the distance, consistent with \cref{thm:no-int-single}. The sign of INT differs by functional role, negative for L9H9 (Name Mover) and positive for L8H6 (S-Inhibition); we expand on this pattern in Section~\ref{sec:ioi-case-study}.

\begin{wrapfigure}{r}{0.46\linewidth}
  \centering
  \vspace{-1em}
  \includegraphics[width=\linewidth]{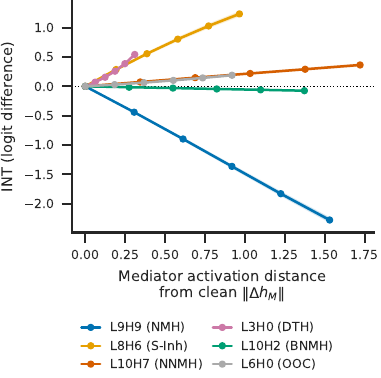}
  \caption{Mean INT for six IOI heads in GPT-2 small with different roles, as each head's patched activation is interpolated from its clean to its counterfactual value. $|\interaction|$ grows approximately linearly with distance, consistent with \cref{thm:no-int-single}.}
  \label{fig:int-distance}
  \vspace{-1em}
\end{wrapfigure}

Condition~(ii) of \cref{thm:no-int-single} explains why INT is mostly zero in transformers compiled with Tracr \citep{lindner2023tracr}, which are used to evaluate circuit discovery methods \citep{conmy2023towards, gupta2025interpbenchsemisynthetictransformersevaluating}. Tracr stores each program variable in its own orthogonal subspace of the residual stream, and each operation reads only the variables it is given. Patching a mediator therefore affects only the operations that use its output, and INT can be nonzero only if one of them also uses something else that the corruption changes, for example by comparing two counts. Otherwise the effects of the mediator and its bypass simply add, making $\Psi$ additively separable on $P$. 

Across all 129 Tracr programs in the InterpBench repository \citep{gupta2025interpbenchsemisynthetictransformersevaluating}, $|\mathrm{INT}| < 10^{-5}$ for every mediator and input pair in 99 of them. In the other 30, INT is still zero for nearly half of the mediator and input pairs at which patching the mediator changes the output. Compiled model benchmarks likely do not reflect the INT conditions present in real transformers. The magnitude of INT is determined by the evaluation conditions, including model class and counterfactual prompt design, and not by the circuit alone.

\section{Multi-Component Interactions}
\label{sec:interaction-multi}

Single-component patching aims to measure causal attribution for each head in isolation, but nearly every mechanistic interpretability study then uses those scores to select or ablate \emph{sets} of components. Circuit faithfulness evaluation, grouping nodes with a common function, and combinatorial knockouts are all examples of multi-component patching operations in common practice. We show that in this setting, interaction effects appear not just at every patched component individually, but also among every pair and subset of components as well --- meaning that individual patching scores need not predict what happens when those components are patched together.

\begin{restatable}[Multi-Component Interaction]{theorem}{multiMediatorInt}
\label{thm:multi-noising-decomp}
For a subset of model components $\mathcal{C} \subseteq \mathcal{M}$, let $Y(\mathcal{C})$ denote the output of the multi-component noising estimator that runs the model on treatment $a$ and patches $H_i \mapsto H_i(a')$ for all $H_i \in \mathcal{C}$, with $Y(\emptyset) = Y(a)$. The causal effect computed by the multi-component noising estimator decomposes as:
    \begin{equation}
        Y(\emptyset) - Y(\mathcal{C})
        \;=\;
        \sum_{M \in \mathcal{C}} \underbrace{\bigl[Y(\emptyset) - Y(M)\bigr]}_{\mathrm{NIE}(M)}
        \;+\;
        \sum_{\substack{\mathcal{T} \subseteq \mathcal{C} \\ |\mathcal{T}| \geq 2}} \mathrm{xINT}(\mathcal{T}),
        \label{eq:multi-interaction-decomp-partial}
    \end{equation}
    where $\mathrm{xINT}(\mathcal{T}) = - \sum_{\mathcal{R} \subseteq \mathcal{T}} (-1)^{|\mathcal{T}|-|\mathcal{R}|} Y(\mathcal{R})$
    is the cross-component interaction among each subset of nodes $\mathcal{T}$. Each $\nie$ term can be computed with the single-component noising estimator, or further decomposed as
    \begin{equation}
    Y(\emptyset) - Y(\mathcal{C})
    \;=\;
    \underbrace{\sum_{M \in \mathcal{C}} \pie(M)}_{\substack{\text{path-specific} \\ \text{effects}}}
    \;+\;
    \underbrace{\sum_{M \in \mathcal{C}} \interaction(M)}_{\substack{\text{individual} \\ \text{interactions}}}
    \;+\;
    \underbrace{\sum_{\substack{\mathcal{T} \subseteq \mathcal{C} \\ |\mathcal{T}| \geq 2}} \mathrm{xINT}(\mathcal{T})}_{\substack{\text{cross-component} \\ \text{interactions}}}.
        \label{eq:multi-interaction-decomp-full}
    \end{equation}
\end{restatable}

Analogous to the single-mediator decomposition $\nie = \pie + \interaction$, the multi-mediator decomposition reads as $\nie = \sum \pie + \sum \mathrm{INT} + \sum \mathrm{xINT}$, where $\mathrm{INT}$ and $\mathrm{xINT}$ are the individual interactions and cross-interactions respectively. The full proof is in Appendix~\ref{app:proofs}.

Of the three terms, xINT is invisible to per-head measurement and arises specifically when multiple nodes are patched, as it measures how components interact with each other when patched together. Since each component's output is added to the residual stream, the counterfactual activations that are patched at different points in the model compound through shared nonlinearities in a way that individual patches do not. As a result, measuring each head's NIE, PIE, and INT individually during the localization and ranking process is still not enough to predict what a circuit does collectively.

\paragraph{Faithfulness and Knockouts}
Faithfulness is the most common metric for evaluating a circuit: it measures how much of the full model's behavior survives when every component outside the circuit is ablated. Following \citet{wang2023interpretability}, who measure it by the gap between the model's and the circuit's performance, we write faithfulness as a difference: $F = Y(\emptyset) - Y(\bar{\mathcal{C}})$, where $\mathcal{C} \subseteq \mathcal{M}$ is the circuit and $\bar{\mathcal{C}} = \mathcal{M} \setminus \mathcal{C}$ is the set of components outside it. $F = 0$ when the circuit reproduces the model's output, and a larger $F$ means a less faithful circuit. Since $F = \nie(\bar{\mathcal{C}})$, i.e. the noising effect of patching the set $\bar{\mathcal{C}}$, \cref{thm:multi-noising-decomp} applies to it directly.

Knockouts of the IOI circuit in GPT-2 small show a cross-interaction directly. On the 1000 prompt pairs of \cref{sec:ioi-case-study}, with each pair's $p_{ABC}$ prompt as $a'$ and every head patched at all positions, ablating all 144 heads lowers the logit difference by $3.68 \pm 0.15$ (mean over pairs, 95\% CI), less than ablating the 26 heads of the circuit alone ($4.75 \pm 0.14$). This holds for individual prompts as well as on average: on 860 of the 1000 pairs, ablating all heads lowers the logit difference less than ablating the circuit alone. \citet{shi2024hypothesis} made a related observation for edge-level knockouts on the IOI circuit, where knocking out the complete model had less impact than knocking out the circuit.

Treating the circuit and the heads outside it as two component sets,
\cref{thm:multi-noising-decomp} gives $Y(\emptyset) - Y(\mathcal{C} \cup \bar{\mathcal{C}}) = \mathrm{NIE}(\mathcal{C}) + \mathrm{NIE}(\bar{\mathcal{C}}) + \mathrm{xINT}(\mathcal{C}, \bar{\mathcal{C}})$. Ablating the heads outside the circuit lowers the logit difference on its own, $\mathrm{NIE}(\bar{\mathcal{C}}) = F = 1.44 \pm 0.08$, but the cross-interaction is negative and larger in magnitude, $\mathrm{xINT}(\mathcal{C}, \bar{\mathcal{C}}) = -2.52 \pm 0.08$, so ablating all heads costs less than ablating the circuit alone. With the circuit intact, ablating the heads outside it lowers the logit difference by $1.44$; with the circuit ablated, ablating them raises it by $1.08$. The same group of heads thus helps the model when the circuit is present and pushes it toward the wrong name when it is absent. The difference between these two effects is the cross-interaction.

\section{When INT is Interpretable: A Case Study on GPT-2}
\label{sec:ioi-case-study}

The Indirect Object Identification circuit in GPT-2 small \citep{wang2023interpretability} is one of the few published circuits whose characterization included combinatorial ablation over groups of attention heads, not just single-head NIE rankings. This was required in order to discover backup mechanisms that only activated when other groups of heads were ablated. This methodology predates the field's widespread adoption of greedy activation patching methods, and it is what makes the circuit's functional labels available as an independent baseline against which hypotheses about interaction effects can be tested. Through our study of interactions in the IOI circuit, we demonstrate the following:

\begin{enumerate}[label=(\alph*), leftmargin=0.7cm]
    \item \textbf{INT values across heads carry genuine mechanistic information.} By clustering heads based on the signs and magnitudes of INT and PIE, we recover heads with backup compensation, context-specific mechanisms, and harmful mechanisms.
    \item \textbf{NIE systematically under-ranks backup mechanisms.} This includes the Backup Name Mover heads as well as the primary Name Mover head with the highest PIE.
    \item \textbf{Removing INT introduces completely different ranking failures.} Ranking by PIE excludes context-specific mechanisms that are required for the Name Mover heads to work effectively, including Previous Token, Duplicate Token, and Induction heads.
\end{enumerate}

All results use $N=1000$ prompt pairs generated with the IOI dataset code of \citet{wang2023interpretability}; each clean prompt is paired with a corrupted prompt from the $p_{ABC}$ distribution, which serves as the baseline $a'$ for computing NIE and PIE. For their knockouts, \citet{wang2023interpretability} used mean-ablation over $p_{ABC}$ instead of a corrupted prompt, which patches in an activation that no input produces and therefore defines no PIE or INT. Further details can be found in Appendix~\ref{app:ioi-case-study}. 

\subsection{INT Patterns Across Functional Head Groups}
\label{sec:ioi-roles}

\begin{figure*}[t]
  \centering
  \includegraphics[width=\linewidth]{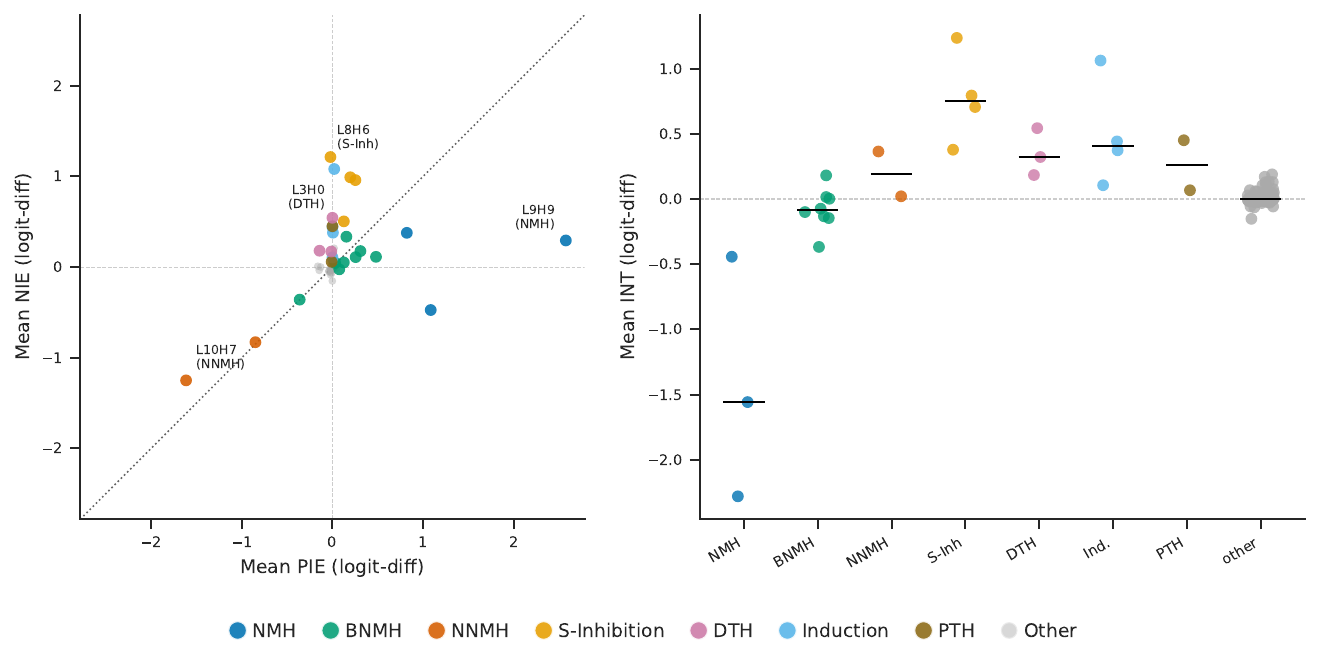}
  \caption{%
    \textbf{Left:} Mean NIE vs.\ mean PIE per head, colored by their functional group from \citet{wang2023interpretability}.
    NMH and most BNMH sit below the diagonal; S-Inh, DTH, PTH, and Induction cluster near the $y$-axis (PIE $\approx 0$, $\interaction \approx \nie$); NNMH occupies the negative-PIE region above the diagonal.
    \textbf{Right:} Mean $\interaction$ by functional group; each point is one head, with median INT marked with black lines.
    NMH and most BNMH are negative; S-Inh, DTH, PTH, Induction, and NNMH are positive; the remaining heads form a median-zero cluster.
  }
  \label{fig:ioi-dist}
\end{figure*}

We describe three mechanistically distinct interaction patterns that emerge across the attention heads in GPT-2 small on the IOI task. The NIE, PIE, and INT values are shown in \cref{fig:ioi-dist}.

\paragraph{Backup compensation.}
These are heads with $\interaction < 0$ and $\nie < \pie$; their marginal contribution shrinks when the rest of the circuit sees the clean input. This pattern is mainly seen in the Name Mover Heads (NMH) and the Backup Name Mover Heads (BNMH) that behave like them. Most notably, the name mover head L9H9 has the highest $\pie = +2.575$ but nearly equal and opposite $\interaction = -2.281$, resulting in a much lower ranking $\nie = +0.294$.

\paragraph{Context-specific mechanisms.}
These are heads with $\interaction > 0$ and $\pie \approx 0$; their contribution only materializes when specific downstream components are active. This pattern is mainly seen in the Duplicate Token (DTH), Induction, and Previous Token (PTH) heads; \cref{fig:ioi-dist} shows that their PIE values are small and they contribute mainly through INT. S-Inhibition heads also fit this pattern, though they have slightly higher $\pie$ ($-0.02$ to $+0.25$).

\paragraph{Harmful mechanisms.}
These are heads with $\pie < 0$ and $\interaction > 0$; they directly hurt task performance but are partially opposed by other heads when they are active. This pattern is seen in the two Negative Name Mover heads (NNMH), which have PIEs of $-0.8$ and $-1.6$ and write the wrong answer to the output.

As shown in \cref{fig:ioi-dist} (right), these patterns are largely stable within head groups and structured by sign and magnitude. They also sharpen the AND/OR circuit picture of \citet{heimersheim2024patching}: the same INT sign can result from qualitatively different functional relationships with the other circuit heads. Together, they demonstrate that INT carries information about functional roles that neither NIE nor PIE recover on their own.

\subsection{Diagnosing Activation Patching Results}
\label{sec:ioi-failures}

We showed in \cref{sec:interaction-single} that NIE and PIE rankings on the IOI task in GPT-2 disagree substantially. Here we show that this disagreement is not arbitrary; each estimand systematically under-ranks a different functional group of heads in the circuit. The rankings are shown in Figure~\ref{fig:rankings}.

\begin{figure*}[t]
  \centering
  \includegraphics[width=\linewidth]{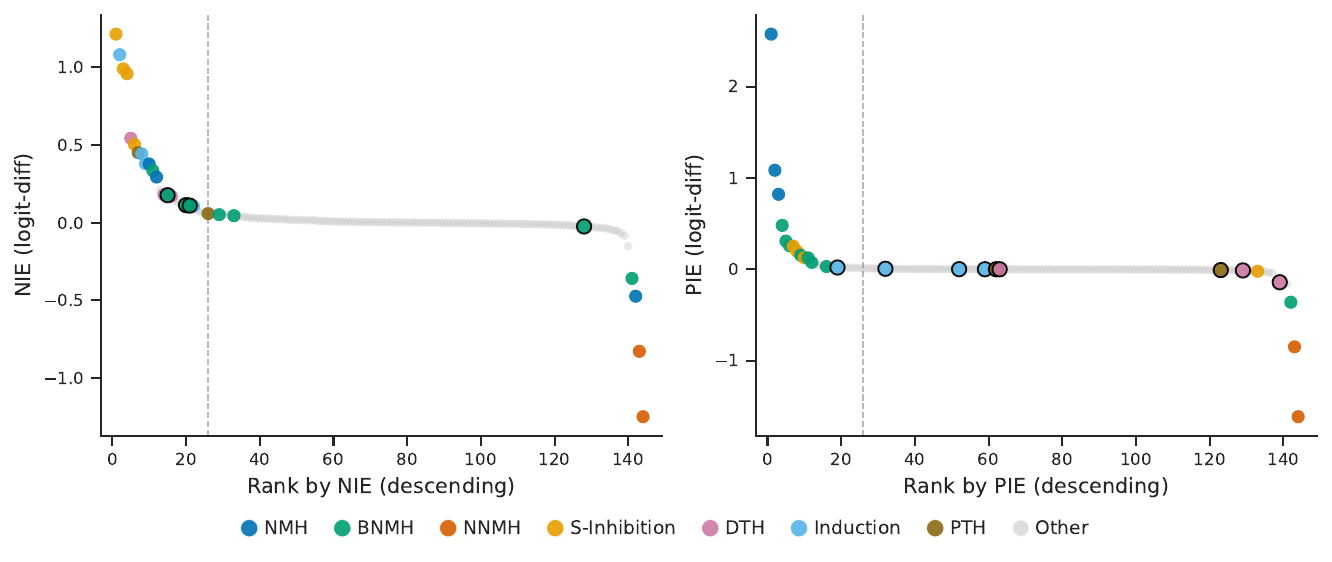}
  \caption{%
    All 144 heads ranked by signed mean (descending); dashed vertical line marks the Wang et al.\ circuit boundary ($k=26$).
    Outlined circles identify the heads each estimand misranks.
    \textbf{(Left)} NIE ranking: the four backup name mover heads that \citet{wang2023interpretability} describe as behaving like Name Mover Heads (outlined) fall from PIE ranks 4--12 to NIE ranks 15--128.
    \textbf{(Right)} PIE ranking: heads that contribute almost entirely through INT (DTH, Induction, PTH, outlined) fall to ranks 19--139, with only L5H5 inside the circuit boundary, even though NIE places all nine within the top 26.
  }
  \label{fig:rankings}
\end{figure*}

\paragraph{NIE underranks the backup tier.}

\citet{wang2023interpretability} chose eight backup heads whose direct effect on logit difference, after knocking out the NMH group, exceeded a 2\% threshold, and found that only four of them behave similarly to Name Mover Heads. We find these four heads have $\interaction < 0$, and NIE moves them from PIE ranks 4--12 to ranks 15--128 (\cref{fig:rankings}, left). The name mover L9H9 also falls from rank 1 to 12.

\paragraph{PIE misses context-specific mechanisms.}

At $k = 26$, signed $\nie$ places all nine DTH, Induction, and PTH heads inside the circuit, at ranks 2--26, five of them in its top half. However, PIE rankings exclude all of them except Induction L5H5 ($\pie = +0.020$), which barely makes it in at rank 19. These heads support the S-Inhibition mechanism by identifying and routing positional information about the subject token. Under pABC corruption, however, there is no repeated name in the prompt, so the duplicate-token detection that DTH, Induction, and PTH implement has nothing to operate on.

\subsection{Interactions Between Head Groups}
\label{sec:ioi-xint}

Per-head INT measures a head's interaction with the rest of the model as a whole. Pairwise cross-interaction resolves it into interactions with specific components or groups (\cref{thm:multi-noising-decomp}).  Since $Y(\emptyset) - Y(\{A, B\}) = \nie(A) + \nie(B) + \mathrm{xINT}(A, B)$, the cross-interaction is the change in either head's NIE once the other is already ablated: a negative $\mathrm{xINT}$ means that each head's effect is smaller once the other is ablated, and a positive one means that it is larger.

\paragraph{Serial dependency.} All 36 pairs between the S-Inhibition heads and the DTH, Induction and PTH heads have negative cross-interaction ($-0.004$ to $-0.503$, every 95\% CI excluding zero). Each head is less effective once the other is ablated; the S-Inhibition heads depend on positional information routed to them from these upstream heads. \citet{heimersheim2024patching} refer to such serial dependencies as AND structures.

\paragraph{Two kinds of backup.} The four name-mover-like backup heads have $\interaction < 0$ (\cref{sec:ioi-failures}), which could suggest redundancy with the Name Mover Heads: if they perform the same function, ablating the NMH group should raise these heads' NIE. However, this does not occur for three of them, L10H10, L10H1 and L9H0. Their cross-interaction with the NMH group is at most $0.02$ in magnitude, and each has about the same effect whether or not the NMH group is ablated. For other backup heads it does: L10H2 has small INT ($-0.07$), but its NIE grows from $+0.05$ to $+0.31$ once the NMH group is ablated ($\mathrm{xINT} = +0.26$), and L10H6 and L11H9 grow by smaller amounts ($+0.10$, $+0.07$). BNMH L11H2, which hurts the task in the intact model, switches to helping it ($-0.36$ to $+0.24$) when the NMH group is ablated. Per-head scores cannot identify these types of cross-interactions, and \citet{wang2023interpretability} only found them with combinatorial group knockouts.

\paragraph{Conditional harm.}  The harm of both Negative Name Mover Heads depends on S-Inhibition: all eight S-Inhibition--NNMH pairs have positive cross-interaction, and without S-Inhibition head L8H6 the NIEs of L10H7 and L11H10 reduce by half ($-1.25$ to $-0.62$, $-0.83$ to $-0.41$). L10H7 also depends on the Name Mover Heads: once they are ablated, its NIE becomes positive ($+0.27$), as \citet{wang2023interpretability} observed after their knockout. The context-specific heads and the Negative Name Mover Heads both have $\interaction > 0$ (\cref{sec:ioi-roles}), yet their cross-interactions with the same S-Inhibition heads have opposite signs. Per-head INT cannot distinguish these two relationships, and only the first fits the AND/OR picture of \citet{heimersheim2024patching}.

\subsection{The IOI Circuit as Ground Truth}
\label{sec:ioi-conditions}

The IOI circuit is the closest approximation to ground truth in a realistic language model and the most common benchmark for evaluating localization methods. Its credibility rests on manual discovery rather than automated procedures, and on a complete characterization covering token position-level structure, functional roles, and head groups. Several of the design decisions that produced this characterization are no longer in common practice; we describe below why each was essential for grounding INT values in known functional structure.

\paragraph{Counterfactual prompt design.} The pABC prompt was designed to isolate mechanisms in the model related to duplicate token detection and processing. \cref{fig:int-measurement} shows that under symmetric token replacement corruption, INT collapses and NIE rankings nearly match PIE rankings; NIE would then leave all four Induction heads out of the top 26, including L5H5, which pABC ranks second. The pABC counterfactual is what makes those heads visible, since their contributions exist almost entirely through INT.

\paragraph{Genuine functional modularity.}
GPT-2's IOI circuit implements functionally distinct head groups, and INT signs and magnitudes are consistent within those groups. When the underlying functionality is diffuse across components, INTs between those components reflect mixtures of contributions rather than relationships between distinct mechanisms, making them difficult to interpret. The modularity of the IOI circuit was established independently, and cannot be assumed for arbitrary models and tasks.

\paragraph{Prior independent characterization through combinatorial search.} \citet{wang2023interpretability} identified backup mechanisms through combinatorial ablation over groups of heads. This is what made it possible to confirm that INT values track known functional distinctions rather than simply disagreeing with NIE. Measuring INT is not an escape hatch from combinatorial search, but clustering and interpreting interaction values is a step toward constraining and guiding that search.

\section{Discussion and Future Work}
We have demonstrated that non-zero interaction terms (INTs) are likely widespread in interpretability settings (\S\ref{sec:interaction-single}), and can quantitatively characterize and explain several known failure modes in circuit analysis (\S\ref{sec:ioi-case-study}). Interpretability research thus far has implicitly assumed that components can be explained in a modular fashion, and that counterfactual interventions tend to yield direct estimates of component importance. Nonetheless, if one desires a truly modular description of a component's role, then one must explain and control for INTs. Given how ubiquitous redundancy is in neural networks, overdetermination and preemption \citep{mueller2024missed}, and therefore INTs, seem likely to be frequent. The only methods guaranteed to control for INTs will require full combinatorial or recursive search over mediators, rendering exact methods intractable for contemporary models of interest.

This raises a natural question: \emph{should} one control for INTs? While one's initial instinct might reasonably be to do so, we instead follow \citet{ikram2015proposed} in advocating that INTs be embraced as a useful analysis tool in its own right. Indeed, interaction terms have been directly studied and have yielded notable insights in epidemiology settings \citep{vanderweele2014tutorial,vanderweele2019interaction}. In interpretability, natural analogies could include quantifying to what extent certain mediators' effects are prompt-dependent (i.e., interact with the treatment), or to what extent two mechanisms with no component overlap functionally intersect, and how.

To address the intractability of combinatorial search, future work could explore heuristics-based methods that reduce the search space. For example, much of the cleanliness of analyses of the IOI circuit of \citet{wang2023interpretability} derives from the clustering of components by functional role. Can the discovery of functionally similar components be automated, then? Interpretability agents \citep{shaham2024multimodal,kim2025llmspursueagenticinterpretability,haklay2026pitfallsevaluatinginterpretabilityagents} and more fine-grained taxonomies of model representations \citep[e.g.,][]{arad-etal-2025-saes} seem likely to be informative. Additionally, advances in concept geometry have led to more selective steerability with fewer side effects \citep[e.g.,][]{marks2025sparse,sarfati2026shapebeliefsgeometrydynamics}; while not guaranteed to fully control for INTs, this represents a tractable way to reduce their impact on our interpretation of particular model representations.

\label{sec:discussion}

\ack{We thank Trang Nguyen, Purva Pruthi, Jatin Nainani, and Alice Rigg for feedback and discussions on the paper. S.V. and S.N. were supported by a grant from Coefficient Giving. A.M. was supported by a grant from Coefficient Giving. We declare no competing interests.
}

\bibliography{references}
\bibliographystyle{abbrvnat}

\clearpage
\appendix
\onecolumn

\section{Proofs}
\label{app:proofs}

We reproduce the definition of the path-specific effect \citep{avin2005identifiability} for reference. 

\begin{definition}[Path-Specific Effect]
\label{def:pse}
Fix a DAG $G$ and let $g \subseteq G$ be an \emph{effect subgraph} with complement $\bar{g} = G \setminus g$.
The \emph{modified model} $\mathcal{M}_g$ is obtained from $G$ by replacing each $\bar{g}$-parent of every node $V$ with its baseline value $V(\aprime)$.
The \emph{path-specific effect} along $g$ is
$$\mathrm{PSE}_g(a, \aprime) = \E\bigl[Y_a\big|_{\mathcal{M}_g}\bigr] - \E\bigl[Y_{\aprime}\big|_{\mathcal{M}_g}\bigr]$$
\end{definition}

\singleMediatorInt*

\begin{proof}

Consider a transformer where each layer $\ell$ contains multi-head attention followed by an MLP, such that each layer $\ell$ computes:
$$z_\ell = x_{\ell-1} + \sum_{i = 1}^{n_\mathrm{head}} h^i_\ell(x_{\ell-1}) \qquad
x_\ell = z_\ell + \mathrm{MLP}_\ell(z_\ell)$$
where each $h^i_\ell(\cdot)$ and $\mathrm{MLP}_\ell(\cdot)$ absorb their respective LayerNorms.
Let $M$ be a single attention head $H_\ell^i$ or the MLP, such that the edges that bypass $M$ are given by:
$$B_M = \begin{cases}
\{x_{\ell-1} \to z_\ell\} \cup \{H_\ell^j \to z_\ell : j \neq i\} & \text{if } M = H_\ell^i \\
\{z_\ell \to x_\ell\} & \text{if } M = \mathrm{MLP}_\ell
\end{cases}$$
In models that apply attention and the MLP to the same input in parallel, such as Pythia, the bypass of a head also contains the MLP of its layer, and the argument is unchanged.

We treat the attention head and MLP cases in parallel.  For $M = H_\ell^i$, let
$$r = x_{\ell-1}(a) + \sum_{j \neq i} H_\ell^j(a), \qquad r' = x_{\ell-1}(a') + \sum_{j \neq i} H_\ell^j(a')$$
denote the total bypass contributions to $z_\ell$ under treatment and baseline respectively, so that $z_\ell(a) = r + h$ and $z_\ell(a') = r' + h'$.

For $M = \mathrm{MLP}_\ell$, set $r = z_\ell(a)$ and $r' = z_\ell(a')$, so that $x_\ell(a) = r + h$ and $x_\ell(a') = r' + h'$.

In both cases, let $\Phi(r, h)$ denote the function mapping bypass contribution $r$ and component output $h$ to the outcome $Y$, and write $h = M(a)$, $h' = M(a')$.

The four canonical outcomes are:
$$\Phi(r, h) = Y(a), \quad \Phi(r', h) = Y(a', M(a)), \quad
  \Phi(r, h') = Y(a, M(a')), \quad \Phi(r', h') = Y(a').$$

\paragraph{Part (a).}
Denoising runs the model on $a'$ and patches $M(a') \mapsto M(a)$, producing the counterfactual output $Y(a', M(a))$ by definition.  The denoising effect is therefore
$$Y(a', M(a)) - Y(a') \;=\; \Phi(r', h) - \Phi(r', h') \;=\; \mathrm{PIE}.$$

To identify the corresponding PSE, consider the effect subgraph $g$ with $\bar{g} = B_M$ and all other edges in $g$.  In $\mathcal{M}_g$, all bypass edges are blocked at their baseline values.  For $M = H_\ell^i$, this means $x_{\ell-1}$ is frozen at $x_{\ell-1}(a')$ and every other head $H_\ell^j$ is frozen at $H_\ell^j(a')$, giving bypass contribution $r'$.  For $M = \mathrm{MLP}_\ell$, $z_\ell$ is frozen at $z_\ell(a') = r'$ directly. In both cases, the edge from $M$ into the residual stream is in $g$, so $M\big|_{\mathcal{M}_g} = h$.  

The residual stream at the layer output is therefore $\Phi(r', h)$, and since all downstream edges are in $g$, $Y_a\big|_{\mathcal{M}_g} = \Phi(r', h) = Y(a', M(a))$. Under baseline $a'$, every variable takes its natural value regardless of $g$, so $Y_{a'}\big|_{\mathcal{M}_g} = Y(a')$. Hence $\mathrm{PIE} = \mathrm{PSE}_g(a, a')$ for this subgraph.

\paragraph{Part (b).} 

Noising runs the model on $a$ and patches $M(a) \mapsto M(a')$, producing
the counterfactual output $Y(a, M(a'))$ by definition.  The noising effect
is therefore
$$\mathrm{NIE} \;=\; Y(a) - Y(a, M(a')) \;=\; \Phi(r, h) - \Phi(r, h').$$

The identity $\mathrm{NIE} = \mathrm{PIE} + \mathrm{INT}$ follows by
direct cancellation:
\begin{align*}
\mathrm{PIE} + \mathrm{INT}
  &= \bigl[\Phi(r',h) - \Phi(r',h')\bigr]
   + \bigl[\Phi(r,h) - \Phi(r',h)
           - \Phi(r,h') + \Phi(r',h')\bigr] \\
  &= \Phi(r,h) - \Phi(r,h') \;=\; \mathrm{NIE}
\end{align*}

\paragraph{Part (c).}
The three-way decomposition follows by direct cancellation from the
four-outcome table:
\begin{align*}
\mathrm{PIE} + \mathrm{PDE} + \mathrm{INT}
  &= \bigl[\Phi(r',h) - \Phi(r',h')\bigr]
   + \bigl[\Phi(r,h') - \Phi(r',h')\bigr] \\
  &\quad + \bigl[\Phi(r,h) - \Phi(r',h)
           - \Phi(r,h') + \Phi(r',h')\bigr] \\
  &= \Phi(r,h) - \Phi(r',h') \;=\; Y(a) - Y(a') \;=\; \mathrm{TE}
\end{align*}
To identify PDE as a PSE, consider the effect subgraph $g$ with
$\bar{g} = \{H_\ell^i \to z_\ell\}$ for an attention head and
$\bar{g} = \{\mathrm{MLP}_\ell \to x_\ell\}$ for the MLP, with all
other edges in $g$.  In $\mathcal{M}_g$, all bypass edges are in $g$
so treatment propagates normally through the bypass, giving $\tilde{r}
= r$.  The component edge is in $\bar{g}$, so $M$ is frozen at $h'$.
Hence $Y_a\big|_{\mathcal{M}_g} = \Phi(r, h')= Y(a, M(a'))$, and
$Y_{a'}\big|_{\mathcal{M}_g} = Y(a')$.  Hence $\mathrm{PDE} =
\mathrm{PSE}_g(a, a')$ for this subgraph.
\end{proof}
\noindent\makebox[\linewidth]{\rule{\linewidth}{0.4pt}}

\singleMediatorNoInt*

\begin{proof}
We first express INT in terms of $\Psi$, then write it as the Hessian of $\Psi$ averaged over $P$, which gives the identity and the bound, and finally verify each sufficient condition. Write $h = M(a)$, $h' = M(a')$, $r = B_M(a)$, and $r' = B_M(a')$, so that $\delta_M = h - h'$ and $\delta_B = r - r'$. Since the downstream computation receives $r + h$ as the combined residual stream that $M$ writes to, the four counterfactual outputs from \cref{prop:single-mediator-decomp} are evaluations of $\Psi$ at the four corners of $P$:
\begin{align*}
Y(a)        &= \Psi(r' + h' + \delta_M + \delta_B), &
Y(a', M(a)) &= \Psi(r' + h' + \delta_M), \\
Y(a, M(a')) &= \Psi(r' + h' + \delta_B),            &
Y(a')       &= \Psi(r' + h').
\end{align*}
Substituting into the definition of INT from \cref{prop:single-mediator-decomp}(b):
\begin{equation}
\label{eq:int-psi}
\mathrm{INT} = \Psi(r' + h' + \delta_M + \delta_B) - \Psi(r' + h' + \delta_M) - \Psi(r' + h' + \delta_B) + \Psi(r' + h').
\end{equation}

\paragraph{Real-analyticity of $\Psi$.}
Softmax is a composition of the exponential function and division by a strictly positive sum, hence real-analytic. GeLU is a product of the identity with either the Gaussian CDF or a sigmoid, both of which are real-analytic. LayerNorm divides by $\sqrt{\sigma^2(z) + \epsilon}$, which is real-analytic because $\sigma^2(z) + \epsilon$ is a polynomial in $z$ bounded below by $\epsilon > 0$. For the same reasons, SiLU, and hence SwiGLU, is real-analytic as a product of the identity with a sigmoid, and RMSNorm is real-analytic because it divides by the square root of a polynomial bounded below by $\epsilon > 0$; rotary position embeddings are linear in the queries and keys. Since $\Psi$ is built from these operations and linear maps by composition, sums and products, such as the query-key products and the attention-weighted sums of values, it is real-analytic, and in particular twice continuously differentiable.

\paragraph{Averaged curvature.}
Let $z' = r' + h'$, and define
$$g(s, t) = \Psi\bigl(z(s, t)\bigr), \qquad z(s, t) = z' + s\,\delta_M + t\,\delta_B, \qquad s, t \in [0, 1],$$
so that $z(s, t)$ runs over $P$ and $g$ gives the value of $\Psi$ at each of its points. The corners of $P$ give $g(0,0) = Y(a')$, $g(1,0) = Y(a', M(a))$, $g(0,1) = Y(a, M(a'))$ and $g(1,1) = Y(a)$. Hence \cref{eq:int-psi} becomes
$$\mathrm{INT} = \bigl[g(1,1) - g(0,1)\bigr] - \bigl[g(1,0) - g(0,0)\bigr].$$
The first term is the NIE and the second term is the PIE: each is the change in $g$ as $s$ goes from $0$ to $1$, with $t$ held at $1$ or at $0$.

Since $\Psi$ is twice continuously differentiable, so is $g$, and by the fundamental theorem of calculus each term is an integral over $s$:
$$g(1, t) - g(0, t) = \int_0^1 \partial_s g(s, t)\, ds, \qquad t \in \{0, 1\}.$$
Subtracting the case $t = 0$ from the case $t = 1$,
$$\mathrm{INT} = \int_0^1 \bigl[\partial_s g(s, 1) - \partial_s g(s, 0)\bigr]\, ds.$$
The integrand is now the change in the continuously differentiable function $t \mapsto \partial_s g(s, t)$ as $t$ goes from $0$ to $1$, so a second application of the fundamental theorem of calculus, now in $t$, gives
$$\mathrm{INT} = \int_0^1 \!\! \int_0^1 \partial_t \partial_s\, g(s, t)\; dt\, ds.$$
By the chain rule, $\partial_s g(s, t) = \delta_M^\top \nabla\Psi\bigl(z(s, t)\bigr)$, and differentiating this in $t$ gives the second derivative of $\Psi$ in the directions $\delta_M$ and $\delta_B$:
$$\mathrm{INT} = \int_0^1 \!\! \int_0^1 \delta_M^\top \, \nabla^2\Psi\bigl(z(s, t)\bigr)\, \delta_B \; dt\, ds.$$
Since $\delta_M$ and $\delta_B$ do not depend on $s$ or $t$, they can be taken outside the integral, and since the integrand is continuous on the unit square, the order of integration does not matter. Hence
$$\mathrm{INT} = \delta_M^\top \bar{H}\, \delta_B, \qquad \bar{H} = \int_0^1 \!\! \int_0^1 \nabla^2\Psi\bigl(z(s, t)\bigr)\, ds\, dt,$$
so the interaction is the second derivative of $\Psi$ in the directions of the two perturbations, averaged over $P$.

\paragraph{Curvature bound.}
To bound INT, we bound the integrand of the double integral above. Let $\|\nabla^2\Psi(z)\|$ be the spectral norm of the Hessian at $z$, which is the largest second derivative of $\Psi$ in absolute value along any unit direction at $z$. Since the Hessian of a twice continuously differentiable function is symmetric, it has an orthonormal basis of eigenvectors, and expanding $u$ and $v$ in that basis gives $|u^\top \nabla^2\Psi(z)\, v| \le \|\nabla^2\Psi(z)\|\, \|u\|\, \|v\|$ for all vectors $u$ and $v$. Every point $z(s, t)$ lies in $P$, so for all $s, t \in [0, 1]$
$$\bigl|\delta_M^\top \, \nabla^2\Psi\bigl(z(s, t)\bigr)\, \delta_B\bigr| \;\le\; \|\delta_M\|\, \|\delta_B\|\, \max_{z \in P} \bigl\|\nabla^2\Psi(z)\bigr\|,$$
where the maximum exists by the extreme value theorem, since $P$ is closed and bounded and $\|\nabla^2\Psi(z)\|$ varies continuously with $z$. This bound does not depend on $s$ or $t$, since $\delta_M$ and $\delta_B$ are fixed and the maximum is taken over all of $P$. Since the absolute value of an integral is at most the integral of the absolute value, and the integrand is at most this bound at every point of the unit square,
\begin{align*}
|\mathrm{INT}| &\;\le\; \int_0^1 \!\! \int_0^1 \bigl|\delta_M^\top \, \nabla^2\Psi\bigl(z(s, t)\bigr)\, \delta_B\bigr| \; dt\, ds \\
&\;\le\; \int_0^1 \!\! \int_0^1 \|\delta_M\|\, \|\delta_B\|\, \max_{z \in P} \bigl\|\nabla^2\Psi(z)\bigr\| \; dt\, ds \\
&\;=\; \|\delta_M\|\, \|\delta_B\|\, \max_{z \in P} \bigl\|\nabla^2\Psi(z)\bigr\|,
\end{align*}
where the last step holds because the unit square has area one.

\paragraph{Sufficient conditions.}
Both conditions follow directly from \cref{eq:int-psi}, which uses only the values of $\Psi$ at the four corners of $P$, so neither requires $\Psi$ to be differentiable.

For condition (i), suppose $\delta_M = 0$, so that the mediator takes the same value in all four forward passes. The corners of $P$ then coincide in pairs, $z' + \delta_M + \delta_B = z' + \delta_B$ and $z' + \delta_M = z'$, and \cref{eq:int-psi} becomes
$$\mathrm{INT} = \Psi(z' + \delta_B) - \Psi(z') - \Psi(z' + \delta_B) + \Psi(z') = 0.$$
The case $\delta_B = 0$ is the same with the roles of $M$ and its bypass exchanged.

For condition (ii), suppose $\Psi$ is additively separable on $P$, so that $g(s, t) = f_M(s) + f_B(t)$ for all $s, t \in [0, 1]$. Substituting the four corners into \cref{eq:int-psi},
$$\mathrm{INT} = \bigl[f_M(1) + f_B(1)\bigr] - \bigl[f_M(1) + f_B(0)\bigr] - \bigl[f_M(0) + f_B(1)\bigr] + \bigl[f_M(0) + f_B(0)\bigr] = 0.$$
In particular, if $\Psi$ is affine on $P$, so that $\Psi(z) = c + w^\top z$ for all $z \in P$, then $\Psi$ is separable on $P$ with $f_M(s) = s\, w^\top \delta_M$ and $f_B(t) = c + w^\top z' + t\, w^\top \delta_B$.
\end{proof}
\noindent\makebox[\linewidth]{\rule{\linewidth}{0.4pt}}

\multiMediatorInt*

\begin{proof}

Define the centered set function $f(\mathcal{T}) = Y(\mathcal{T}) - Y(\emptyset)$ for $\mathcal{T} \subseteq \mathcal{C}$, with $f(\emptyset) = 0$. We claim:

\begin{equation}
    f(\mathcal{C}) = - \sum_{\emptyset \neq \mathcal{T} \subseteq \mathcal{C}} \mathrm{xINT}(\mathcal{T})
    \label{eq:centered-set-fn-claim}
\end{equation}

Substituting the definition $\mathrm{xINT}(\mathcal{T}) = - \sum_{\mathcal{R} \subseteq \mathcal{T}} (-1)^{|\mathcal{T}|-|\mathcal{R}|} Y(\mathcal{R})$
and exchanging the order of summation:
\begin{align*}
\sum_{\emptyset \neq \mathcal{T} \subseteq \mathcal{C}} \mathrm{xINT}(\mathcal{T})
&= - \sum_{\emptyset \neq \mathcal{T} \subseteq \mathcal{C}} \;
   \sum_{\mathcal{R} \subseteq \mathcal{T}} (-1)^{|\mathcal{T}|-|\mathcal{R}|} Y(\mathcal{R}) \\
&= - \sum_{\mathcal{R} \subseteq \mathcal{C}} Y(\mathcal{R}) 
   \sum_{\substack{\mathcal{T} : \mathcal{R} \subseteq \mathcal{T} \subseteq \mathcal{C} \\ \mathcal{T} \neq \emptyset}} 
   (-1)^{|\mathcal{T}|-|\mathcal{R}|}
\end{align*}

We evaluate the inner sum by cases. 
For $\mathcal{R} = \mathcal{C}$, the only superset of $\mathcal{C}$ in $\mathcal{C}$ is $\mathcal{C}$ itself, so the inner sum equals $(-1)^0 = 1$, contributing $Y(\mathcal{C})$ to the overall sum. 
For $\emptyset \neq \mathcal{R} \subsetneq \mathcal{C}$, we can drop the constraint $\mathcal{T} \neq \emptyset$ since every superset of a nonempty set is nonempty. 
Letting $k = |\mathcal{C} \setminus \mathcal{R}| \geq 1$ and $j = |\mathcal{T}| - |\mathcal{R}|$:
$$\sum_{j=0}^{k} \binom{k}{j}(-1)^{j} = (1-1)^k = 0,$$
where $\binom{k}{j}$ counts the supersets of $\mathcal{R}$ in $\mathcal{C}$ of size $|\mathcal{R}| + j$. 
For $\mathcal{R} = \emptyset$, the $\mathcal{T} = \emptyset$ term is excluded by the constraint, so the inner sum is:
$$\sum_{j=1}^{k} \binom{k}{j}(-1)^{j} = (1-1)^k - 1 = -1,$$
contributing $-Y(\emptyset)$ to the overall sum. 
Collecting all three cases:
$$ - \sum_{\emptyset \neq \mathcal{T} \subseteq \mathcal{C}} \mathrm{xINT}(\mathcal{T}) = Y(\mathcal{C}) - Y(\emptyset) = f(\mathcal{C})$$
which establishes the claim in \cref{eq:centered-set-fn-claim}.

For singletons $\mathcal{T} = \{M\}$, direct substitution gives $(-1)^0\,Y(\{M\}) + (-1)^1\,Y(\emptyset) = Y(\{M\}) - Y(\emptyset)$. Hence we get:

$$Y(\mathcal{C}) - Y(\emptyset) 
= \sum_{M \in \mathcal{C}} \bigl[Y(\{M\}) - Y(\emptyset)\bigr] 
  - \sum_{\substack{\mathcal{T} \subseteq \mathcal{C} \\ |\mathcal{T}| \geq 2}} 
    \mathrm{xINT}(\mathcal{T}).$$
    
Negating both sides yields \cref{eq:multi-interaction-decomp-partial}.

To recover \cref{eq:multi-interaction-decomp-full}, apply \cref{prop:single-mediator-decomp}(b) to each individual term $Y(\emptyset) - Y(\{M\}) = \mathrm{NIE}(M) = \mathrm{PIE}_M + \mathrm{INT}_M$. The decomposition in \cref{prop:single-mediator-decomp} applies to each $M \in \mathcal{C}$ individually regardless of the other components in $\mathcal{C}$, since $Y(\emptyset) - Y(\{M\})$ is a single-component noising effect. Substituting into the first term of \cref{eq:multi-interaction-decomp-partial} yields \cref{eq:multi-interaction-decomp-full} directly. 

\end{proof}
\noindent\makebox[\linewidth]{\rule{\linewidth}{0.4pt}}

\section{IOI Case Study}
\label{app:ioi-case-study}

All values are from the same dataset of 1000 IOI prompt pairs used in Section~\ref{sec:ioi-case-study}; each pair is a clean prompt and a $p_{ABC}$ prompt with the same template and all names resampled. CIs are mean $\pm 1.96 \times \mathrm{SD}/\!\sqrt{1000}$ across prompt pairs. Signed ranks are over all 144 heads, descending.

\subsection{All Circuit Heads: NIE, PIE, and INT Rankings}
\label{app:ioi-full-table}

\cref{tab:ioi-full} includes the NIE, PIE, and INT values for all the in-circuit nodes in \citet{wang2023interpretability}. Of these numbers, \cref{sec:ioi-case-study} highlights the L9H9 (Name Mover) values, the NIE ranks 15--128 for the four backup heads that behave like Name Mover Heads, the small PIE values for context-specific heads (DTH, Induction, PTH), and the PIE range $-0.02$ to $+0.25$ for S-Inhibition heads. For 22 of the 26 heads, the per-pair SD of PIE exceeds the absolute mean, reflecting large prompt-to-prompt variability. The CIs on means are narrow because $N{=}1000$.

{\footnotesize
\begin{table}[h]
\centering
\caption{%
  Mean NIE, PIE, and INT ($\pm$ 95\% CI, logit-diff units) for all 26 \citet{wang2023interpretability} IOI circuit heads.
  PIE rank / NIE rank = signed rank by descending mean over all 144 heads.
  Fuzzy heads: $\dagger$.
}
\label{tab:ioi-full}
\resizebox{\linewidth}{!}{%
\begin{tabular}{@{}llrrrrrr@{}}
\toprule
Head & Group
  & \multicolumn{1}{c}{PIE $\pm$ CI} & \multicolumn{1}{c}{PIE rank}
  & \multicolumn{1}{c}{NIE $\pm$ CI} & \multicolumn{1}{c}{NIE rank}
  & \multicolumn{1}{c}{INT $\pm$ CI} \\
\midrule
L9H9 & NMH & $+2.575 \pm 0.073$ & 1 & $+0.294 \pm 0.037$ & 12 & $-2.281 \pm 0.066$ \\
L9H6 & NMH & $+1.085 \pm 0.058$ & 2 & $-0.473 \pm 0.029$ & 142 & $-1.558 \pm 0.065$ \\
L10H0 & NMH & $+0.821 \pm 0.065$ & 3 & $+0.378 \pm 0.037$ & 10 & $-0.443 \pm 0.037$ \\
\midrule
L10H10 & BNMH & $+0.481 \pm 0.031$ & 4 & $+0.113 \pm 0.019$ & 20 & $-0.368 \pm 0.020$ \\
L10H6 & BNMH & $+0.309 \pm 0.040$ & 5 & $+0.177 \pm 0.031$ & 15 & $-0.132 \pm 0.014$ \\
L10H1 & BNMH & $+0.257 \pm 0.029$ & 6 & $+0.110 \pm 0.020$ & 21 & $-0.146 \pm 0.013$ \\
L9H7 & BNMH & $+0.154 \pm 0.012$ & 9 & $+0.336 \pm 0.020$ & 11 & $+0.181 \pm 0.015$ \\
L10H2 & BNMH & $+0.126 \pm 0.031$ & 11 & $+0.052 \pm 0.024$ & 29 & $-0.074 \pm 0.013$ \\
L9H0 & BNMH & $+0.076 \pm 0.013$ & 12 & $-0.024 \pm 0.008$ & 128 & $-0.100 \pm 0.011$ \\
L11H9 & BNMH & $+0.030 \pm 0.018$ & 16 & $+0.045 \pm 0.017$ & 33 & $+0.015 \pm 0.003$ \\
L11H2 & BNMH & $-0.361 \pm 0.035$ & 142 & $-0.359 \pm 0.033$ & 141 & $+0.002 \pm 0.005$ \\
\midrule
L11H10 & NNMH & $-0.849 \pm 0.040$ & 143 & $-0.828 \pm 0.039$ & 143 & $+0.021 \pm 0.006$ \\
L10H7 & NNMH & $-1.614 \pm 0.064$ & 144 & $-1.249 \pm 0.057$ & 144 & $+0.365 \pm 0.027$ \\
\midrule
L8H10 & S-Inh & $+0.254 \pm 0.018$ & 7 & $+0.960 \pm 0.036$ & 4 & $+0.706 \pm 0.032$ \\
L7H9 & S-Inh & $+0.197 \pm 0.013$ & 8 & $+0.991 \pm 0.033$ & 3 & $+0.794 \pm 0.031$ \\
L7H3 & S-Inh & $+0.127 \pm 0.014$ & 10 & $+0.505 \pm 0.024$ & 6 & $+0.378 \pm 0.021$ \\
L8H6 & S-Inh & $-0.021 \pm 0.020$ & 133 & $+1.214 \pm 0.060$ & 1 & $+1.236 \pm 0.064$ \\
\midrule
L3H0 & DTH & $+0.000 \pm 0.009$ & 63 & $+0.544 \pm 0.018$ & 5 & $+0.544 \pm 0.020$ \\
L0H10$^\dagger$ & DTH & $-0.012 \pm 0.007$ & 129 & $+0.172 \pm 0.016$ & 16 & $+0.184 \pm 0.016$ \\
L0H1 & DTH & $-0.142 \pm 0.021$ & 139 & $+0.181 \pm 0.025$ & 14 & $+0.322 \pm 0.035$ \\
\midrule
L5H5 & Induction & $+0.020 \pm 0.016$ & 19 & $+1.082 \pm 0.034$ & 2 & $+1.062 \pm 0.037$ \\
L5H9$^\dagger$ & Ind. & $+0.006 \pm 0.006$ & 32 & $+0.380 \pm 0.029$ & 9 & $+0.374 \pm 0.029$ \\
L6H9 & Induction & $+0.001 \pm 0.003$ & 52 & $+0.443 \pm 0.033$ & 8 & $+0.442 \pm 0.033$ \\
L5H8$^\dagger$ & Ind. & $+0.001 \pm 0.003$ & 59 & $+0.106 \pm 0.011$ & 22 & $+0.106 \pm 0.012$ \\
\midrule
L4H11 & PTH & $+0.001 \pm 0.003$ & 62 & $+0.451 \pm 0.023$ & 7 & $+0.451 \pm 0.023$ \\
L2H2 & PTH & $-0.008 \pm 0.002$ & 123 & $+0.059 \pm 0.007$ & 26 & $+0.067 \pm 0.007$ \\
\bottomrule
\end{tabular}%
}
\end{table}
}

\subsection{Backup Name Mover Heads}

\cref{tab:ioi-bnmh} gives NIE, PIE and two cross-interactions for each BNMH head: xINT(L9H9), with the Name Mover head L9H9 alone, and xINT(NMH), with the three Name Mover Heads patched jointly. \cref{sec:ioi-xint} uses the second.

\paragraph{BNMH heterogeneity.}

\citet{wang2023interpretability} observe heterogeneous functionality within the backup name mover group: of the eight heads they label BNMH, four resemble Name Mover Heads, two attend equally to IO and S, one attends more to S1, and one attends to S2 and writes negatively. They chose these heads, in their words arbitrarily, as the eight heads outside other groups whose effect on logit difference after knocking out the Name Mover Heads exceeded a 2\% threshold.

Five BNMH heads show the backup compensation pattern of Section~\ref{sec:ioi-roles} ($\interaction < 0$, $\nie < \pie$): L10H10, L10H1, L10H6, L9H0 and L10H2. Three do not. L11H2 has mean PIE $= -0.361$, consistent with the observation of \citet{wang2023interpretability} that it attends more to S1 and writes in the direction of S; it is the only BNMH outside the signed PIE top-26, and its INT is near zero ($+0.002 \pm 0.005$). L11H9's INT is small and positive ($+0.015 \pm 0.003$). L9H7 has clearly positive INT ($+0.181 \pm 0.015$), so $\nie > \pie$; \citet{wang2023interpretability} describe it as attending to S2 and writing negatively. \cref{sec:ioi-xint} shows that which of these heads respond to ablating the Name Mover Heads does not follow the sign of their INT.

{\footnotesize
\begin{table}[h]
\centering
\caption{%
  BNMH heads: estimand values ($\pm$ 95\% CI), signed PIE/NIE ranks,
  pairwise cross-interaction with L9H9, cross-interaction with the full NMH group,
  and \citet{wang2023interpretability} behavioral profile.
  All logit-diff units; $N=1000$; NIE and xINT by resample from each pair's pABC prompt.
}
\label{tab:ioi-bnmh}
\resizebox{\linewidth}{!}{%
\begin{tabular}{@{}lrrrrrrp{3cm}@{}}
\toprule
Head & PIE $\pm$ CI & PIE rank & NIE $\pm$ CI & NIE rank & xINT(L9H9) $\pm$ CI & xINT(NMH) $\pm$ CI & \citet{wang2023interpretability} profile \\
\midrule
L10H10 & $+0.481 \pm 0.031$ & 4 & $+0.113 \pm 0.019$ & 20 & $+0.050 \pm 0.009$ & $-0.015 \pm 0.016$ & name-mover-like \\
L10H6 & $+0.309 \pm 0.040$ & 5 & $+0.177 \pm 0.031$ & 15 & $+0.060 \pm 0.007$ & $+0.097 \pm 0.013$ & name-mover-like \\
L10H1 & $+0.257 \pm 0.029$ & 6 & $+0.110 \pm 0.020$ & 21 & $+0.008 \pm 0.005$ & $+0.017 \pm 0.009$ & name-mover-like \\
L9H7 & $+0.154 \pm 0.012$ & 9 & $+0.336 \pm 0.020$ & 11 & $+0.040 \pm 0.004$ & $+0.028 \pm 0.006$ & attends to S2; writes negatively \\
L10H2 & $+0.126 \pm 0.031$ & 11 & $+0.052 \pm 0.024$ & 29 & $+0.146 \pm 0.009$ & $+0.260 \pm 0.015$ & attends equally to IO and S; writes both \\
L9H0 & $+0.076 \pm 0.013$ & 12 & $-0.024 \pm 0.008$ & 128 & $-0.014 \pm 0.003$ & $+0.009 \pm 0.004$ & name-mover-like \\
L11H9 & $+0.030 \pm 0.018$ & 16 & $+0.045 \pm 0.017$ & 33 & $+0.017 \pm 0.005$ & $+0.066 \pm 0.008$ & attends equally to IO and S; writes both \\
L11H2 & $-0.361 \pm 0.035$ & 142 & $-0.359 \pm 0.033$ & 141 & $+0.008 \pm 0.011$ & $+0.596 \pm 0.044$ & attends to S1; writes in $W_U[S]$ direction \\
\bottomrule
\end{tabular}%
}
\end{table}
}

\subsection{Pairwise Cross-Interaction of S-Inhibition Heads}
\label{app:ioi-pairwise}

\cref{tab:sinh-xint} gives the cross-interaction of each S-Inhibition head with each DTH, Induction, PTH and NNMH head. All 36 pairs with DTH, Induction and PTH heads are negative and all 8 pairs with NNMH heads are positive; every 95\% CI excludes zero. The largest values in both directions involve L8H6: $-0.503$ with Induction head L5H5 and $+0.630$ with NNMH L10H7. \cref{sec:ioi-xint} interprets both signs.

\begin{table}[h]
\centering
\caption{%
  Pairwise cross-interaction xINT between each S-Inhibition head (rows) and each DTH, Induction, PTH and NNMH head (columns), mean over $N{=}1000$ pairs, logit-diff units, resample from the pABC prompt. Every 95\% CI excludes zero; the largest half-width is $0.032$. Fuzzy heads: $\dagger$.
}
\label{tab:sinh-xint}
\resizebox{\linewidth}{!}{%
\begin{tabular}{@{}l*{11}{r}@{}}
\toprule
 & \multicolumn{3}{c}{DTH} & \multicolumn{4}{c}{Induction} & \multicolumn{2}{c}{PTH} & \multicolumn{2}{c}{NNMH} \\
\cmidrule(lr){2-4} \cmidrule(lr){5-8} \cmidrule(lr){9-10} \cmidrule(l){11-12}
S-Inh head & L0H1 & L3H0 & L0H10$^\dagger$ & L5H5 & L6H9 & L5H8$^\dagger$ & L5H9$^\dagger$ & L2H2 & L4H11 & L10H7 & L11H10 \\
\midrule
L7H3 & $-0.012$ & $-0.064$ & $-0.043$ & $-0.067$ & $-0.025$ & $-0.046$ & $-0.023$ & $-0.004$ & $-0.037$ & $+0.173$ & $+0.112$ \\
L7H9 & $-0.049$ & $-0.099$ & $-0.051$ & $-0.240$ & $-0.131$ & $-0.072$ & $-0.099$ & $-0.012$ & $-0.127$ & $+0.366$ & $+0.249$ \\
L8H6 & $-0.050$ & $-0.167$ & $-0.060$ & $-0.503$ & $-0.295$ & $-0.087$ & $-0.162$ & $-0.031$ & $-0.252$ & $+0.630$ & $+0.420$ \\
L8H10 & $-0.036$ & $-0.116$ & $-0.027$ & $-0.203$ & $-0.069$ & $-0.018$ & $-0.118$ & $-0.013$ & $-0.079$ & $+0.447$ & $+0.220$ \\
\bottomrule
\end{tabular}%
}
\end{table}

\section{Greater-Than Case Study}
\label{app:gt-xint}

We applied pairwise cross-interaction analysis (\cref{thm:multi-noising-decomp}) to the early MLPs of the Greater-Than circuit of \citet{hanna2023how}, and found that MLP2 partly compensates for MLP1 and MLP3. Unlike \citet{wang2023interpretability}, \citet{hanna2023how} did not perform combinatorial ablations, so there is no existing pairwise ground truth for comparison.

\paragraph{Setup.} Each prompt has the form ``The \{event\} lasted from the year $\mathtt{XXYY}$ to the year $\mathtt{XX}$'', and its corrupted version replaces $\mathtt{YY}$ with $\mathtt{01}$; we use the 1000 prompt pairs released by \citet{hanna2024have}. The metric is the mean logit of the year endings greater than $\mathtt{YY}$ minus the mean logit of those up to $\mathtt{YY}$. We ablate an MLP by replacing its output at each position with its mean over the year-$\mathtt{01}$ prompts, and compute $\mathrm{xINT}(A, B) = Y(\{A\}) + Y(\{B\}) - Y(\{A, B\}) - Y(\emptyset)$ for each pair of MLPs $A$ and $B$. A negative value means that each MLP is less effective once the other is ablated; a positive value means that it is more effective.

\paragraph{Results.} \citet{hanna2023how} describe the early MLPs as a chain in which each MLP relies on those before it, except MLP2 on MLP1. If the MLPs formed only a chain, ablating an upstream MLP would remove part of the input of the MLPs after it, so the cross-interaction of each pair would be negative. This holds for MLP0 with each of MLPs 1--3 and for MLP1 with MLP3, but MLP2 has positive cross-interaction with both MLP1 and MLP3 (\cref{tab:gt-within}): its function overlaps with each, so each matters more once the other is ablated, as several backup heads of the IOI circuit do for the Name Mover Heads (\cref{sec:ioi-xint}). 

Scoring MLPs one at a time therefore understates how much they matter together: ablating MLP1 alone lowers the logit difference by $0.374$ and MLP2 alone by $0.517$, but ablating both lowers it by $1.107$, and the difference is their cross-interaction (\cref{thm:multi-noising-decomp}). The same pattern holds under resample ablation from each pair's year-$\mathtt{01}$ prompt.

\begin{table}[h]
\centering
\caption{%
  Pairwise cross-interaction xINT among the early MLPs of the Greater-Than circuit in GPT-2 small, under mean ablation over the year-01 prompts and under resample ablation from each pair's year-01 prompt (mean $\pm$ 95\% CI over 1000 prompt pairs, logit-difference units).
}
\label{tab:gt-within}
\begin{tabular}{@{}lrr@{}}
\toprule
Pair & Mean ablation & Resample \\
\midrule
MLP0 $\times$ MLP1 & $-0.286 \pm 0.018$ & $-0.282 \pm 0.014$ \\
MLP0 $\times$ MLP2 & $-0.391 \pm 0.021$ & $-0.429 \pm 0.017$ \\
MLP0 $\times$ MLP3 & $-0.361 \pm 0.015$ & $-0.394 \pm 0.011$ \\
MLP1 $\times$ MLP2 & $+0.216 \pm 0.013$ & $+0.194 \pm 0.011$ \\
MLP1 $\times$ MLP3 & $-0.039 \pm 0.011$ & $-0.057 \pm 0.008$ \\
MLP2 $\times$ MLP3 & $+0.134 \pm 0.013$ & $+0.087 \pm 0.010$ \\
\bottomrule
\end{tabular}
\end{table}

\end{document}